\documentclass[a4paper,10pt]{article}%centerh1

\usepackage[dvipsnames]{xcolor}
\usepackage[sans]{dsfont}
\usepackage[mathscr]{euscript}
\usepackage[centertags]{amsmath}
\usepackage{cite}
\usepackage{amsfonts}
\usepackage{amsthm,enumerate,amssymb}
\usepackage{graphicx}

\usepackage{hyperref}

\newcommand{\numberset}{\mathbb}
\newcommand{\R}{\numberset{R}}

\newtheorem{thm}{Theorem}[section]

\newtheorem{prop}[thm]{Proposition}

\newtheorem{defin}[thm]{Definition}

\begin{document}

\title{A note on the article \textit{On Exploiting Spectral Properties for Solving
MDP with Large State Space}}

\author{D. Maran}

\maketitle
\begin{abstract}
	We improve a theoretical result of the article \textit{On Exploiting Spectral Properties for Solving
MDP with Large State Space} showing that their algorithm, which was proved to converge under some unrealistic assumptions, is actually guaranteed to converge always.
\end{abstract}

\section{Introduction}
The article \cite{Liu} introduces a method to generalize the value iteration algorithm, which becomes computationally unfeasible for MDPs with large state-space. This algorithm requires run the value iteration algorithm on a subspace of the state space that is chosen according to the spectral properties of the probability transition matrix of the process.

The only theoretical guarantee of the article, Proposition 1,  assert that the algorithm converges if the discount factor $\alpha$ is in the range
$$\bigg [0,\frac{1}{\sqrt{N}}\bigg )\qquad N:=|S|.$$

Unfortunately, this guarantee is useless, since with a huge state space, say $N=10^8$, $\alpha$ should be smaller than $\frac{1}{\sqrt{N}}=10^{-4}$, while it is well-known that in most applications the discount factor $\alpha$ should vary in the range $[0.95,1]$. In this short letter, we are showing that, fortunately, the algorithm they propose works also for every $\alpha \in [0,1)$.

\section{Spectral radius}

Let us recall some definitions of linear algebra.

\begin{defin}
	Given a matrix $A\in \R^{n\times n}$ we define its spectral radius as
	$$\rho(A):= \text{max} \{|\lambda_1(A)|, |\lambda_2(A)|, |\lambda_3(A)|,...\}$$
	(where $\lambda_i(A)$ stands for the eigenvalues of $A$)
\end{defin}

Since orthogonal transformation do not modify the non-zero elements of the spectrum, it is straightforward to prove the following
\begin{prop}\label{base}
	Given a matrix $A\in \R^{n\times n}$ and an orthogonal matrix $U\in \R^{n\times k}$, we have
	$$\rho(A) = \rho(U^T A U)$$
\end{prop}

In order to complete the following proof, we will need another property that can be found in \cite{Qua} as theorem 1.5 of pag. 26.

\begin{prop}\label{qua}
	Given a matrix $A\in \R^{n\times n}$ we have
	$$\rho(A)<1 \iff \lim_k A^k=0$$
\end{prop}

\section{Result on the spectral value iteration algorithm}

Using the notation of \cite{Liu}, we are calling
\begin{itemize}
	\item $P_\mu$ is the state transition matrix under policy $\mu$, such that $\|P_\mu\|_\infty = 1$.
	\item $U_K$ is the projection matrix over the basis $u_1,...u_K$, where we want to project the state-value function, so it is orthogonal.
	\item $\underline c_\mu:=c(s, \mu(s))$ is the vector of the expected costs corresponding to each state using policy $\mu$.
	\item $\tilde V_\mu^{(k)}$, which is the $k-$th iteration of the algorithm, is given by
		$$\begin{cases}\tilde V_\mu^{(0)} = 0 \\ \tilde V_\mu^{(k+1)} = U_K^T\underline c_\mu + \alpha U_K^T P_\mu U_K \tilde V_\mu^{(k)}.\end{cases}$$
	This iteration is referred as algorithm (2) in the article \cite{Liu}.
\end{itemize}

For the definitions of $\tilde V_\mu, U_K, P_\mu$ see the main article \cite{Liu}.
\begin{prop}
    The iterative method in algorithm (2) transform domain converges for every $\alpha<1$
\end{prop}
\begin{proof}
    
    Define the matrix $A:=U_K^T P_\mu U_K$. We have
    $\rho(A)\le 1$, indeed, since $U_K$ is orthogonal, we have from \ref{base}
    $$\rho(A)= \rho(P_\mu).$$
    Moreover, we have that 
    $$\rho(P_\mu)\le \|P_\mu\|_\infty = 1$$

    since, being $\lambda_0$ any eigenvalue, and $v$ its corresponding infinity norm normalized eigenvector, we have
    $$|\lambda_0| = \|\lambda_0 v\|_\infty =\|P_\mu v\|_\infty \le \|v\|_\infty=1.$$
    Therefore, $\rho(\sqrt \alpha A)\le \sqrt \alpha <1$, and so, by \ref{qua}, we have also
    $$\big (\sqrt \alpha A \big)^k \to 0.$$

    Moreover, since a convergent sequence is always bounded, we can take $M>0$ such that 
    $$\forall k,\ \|\big (\sqrt \alpha A \big)^k\|_2 \le M.$$
    At this point, by definition of the algorithm, 
    $$\tilde V_\mu^{(k+1)} = U_K^T\underline c_\mu + \alpha A \tilde V_\mu^{(k)}=\sum_{i=0}^{k} \alpha^i A^i U_K^T\underline c_\mu$$
    in order to prove that this series is convergent with $k\to +\infty$, we can show that it is absolutely convergent:
    $$\sup_k \sum_{i=0}^{k} \|  \alpha^i A^i U_K^T\underline c_\mu  \|_2\le \sup_k \sum_{i=0}^{k} \|  \alpha^i A^i   \|_2\|\underline c_\mu\|_2\le \sup_k \sum_{i=0}^{k} \alpha^\frac{i}{2}\|  \big (\sqrt\alpha A\big )^i   \|_2\|\underline c_\mu\|_2 \le $$
    
    $$\le \sup_k \sum_{i=0}^{k} \alpha^\frac{i}{2} M \|\underline c_\mu\|_2 \le \frac{M}{1-\sqrt \alpha}\|\underline c_\mu\|_2 <\infty $$
    this implies that the sequence $\tilde V_\mu^{(k)} $ converges.
\end{proof}

This shows that algorithm (2) always converges. Furthermore, note that, from Gelfand's Formula, we have
\begin{thm}
	Given a matrix $A\in \R^{n\times n}$, we have
	$$\lim_k \|A^k\|_2^{1/k}=\lim_k \|A^k\|_\infty ^{1/k}=\rho(A)$$
\end{thm}
Thus, since, as stated before,  $\rho(U_K^T P_\mu U_K)= \rho(P_\mu)$, there is no difference between the asymptotic speed of convergence with respect to $k$ between the regular value iteration method and the one of algorithm (2).

\end{document}